\documentclass[letterpaper]{article} 
\usepackage{aaai2026}  
\usepackage{times}  
\usepackage{helvet}  
\usepackage{courier}  
\usepackage[hyphens]{url}  
\usepackage{graphicx} 
\usepackage{natbib}  
\usepackage{caption} 
\usepackage{algorithm}
\usepackage{algorithmic}
\usepackage{dsfont}
\usepackage{amsmath}
\usepackage{ntheorem}
\usepackage{amssymb}
\usepackage{booktabs}
\usepackage{multirow}
\usepackage{subfigure}
\usepackage{algorithm}
\usepackage{algorithmic}
\usepackage{adjustbox}
\usepackage{makecell}
\newtheorem{theorem}{Theorem}
\newtheorem{lemma}{Lemma}
\newtheorem{definition}{Definition}
\newtheorem*{proof}{Proof}

\usepackage{newfloat}
\usepackage{listings}
\DeclareCaptionStyle{ruled}{labelfont=normalfont,labelsep=colon,strut=off} 
\floatstyle{ruled}
\newfloat{listing}{tb}{lst}{}
\floatname{listing}{Listing}
\title{FairGSE: Fairness-Aware Graph Neural Network without High False Positive Rates}
\author{
    Zhenqiang Ye\textsuperscript{\rm 12}\equalcontrib, 
     Jinjie Lu\textsuperscript{\rm 12}\equalcontrib, 
     Tianlong Gu\textsuperscript{\rm 12}\thanks{Corresponding author.}, 
     Fengrui Hao\textsuperscript{\rm 12}, 
    Xuemin Wang\textsuperscript{\rm 3}
}
\affiliations{
    \textsuperscript{\rm 1}College of Cyber Security, Jinan University\\
    \textsuperscript{\rm 2}Engineering Research Center of Trustworthy AI (Ministry of Education)\\
    \textsuperscript{\rm 3}Guangxi Key Laboratory of Trusted Software

    yzq66f@stu.jnu.edu.cn, lujinjie0424@stu2023.jnu.edu.cn, gutianlong@jnu.edu.cn, \\
    haofengrui@stu2021.jnu.edu.cn, wangxuemin@guet.edu.cn
}

\usepackage{bibentry}

\begin{document}

\maketitle

\begin{abstract}
Graph neural networks (GNNs) have emerged as the mainstream paradigm for graph representation learning due to their effective message aggregation. However, this advantage also amplifies biases inherent in graph topology, raising fairness concerns. Existing fairness-aware GNNs provide satisfactory performance on fairness metrics such as Statistical Parity and Equal Opportunity while maintaining acceptable accuracy trade-offs. Unfortunately, we observe that this pursuit of fairness metrics neglects the GNN's ability to predict negative labels, which renders their predictions with extremely high False Positive Rates (FPR), resulting in negative effects in high-risk scenarios. To this end, we advocate that classification performance should be carefully calibrated while improving fairness, rather than simply constraining accuracy loss. Furthermore, we propose Fair GNN via Structural Entropy (\textbf{FairGSE}), a novel framework that maximizes two-dimensional structural entropy (2D-SE) to improve fairness without neglecting false positives. Experiments on several real-world datasets show FairGSE reduces FPR by 39\% vs. state-of-the-art fairness-aware GNNs, with comparable fairness improvement.
\end{abstract}


\section{Introduction}

Graph neural networks (GNNs) have developed as one of the most prominent techniques for modeling graph-structured data. The outstanding performance of GNNs can be attributed to the message aggregation mechanism \cite{ying2019gnnexplainer}, which iteratively updates a node representation by aggregating the feature information from its neighbors. However, recent studies \cite{NIFTY, FairGNN} have shown that message aggregation can introduce biases inherent in graph topology. Such biases inflict discrimination on specific sensitive groups, which are represented by subsets of nodes sharing the same sensitive attribute (e.g., gender \cite{genderbias}, race \cite{NIFTY}), raising group fairness\begin{figure}[htbp]
\centerline{\includegraphics[scale=0.30]{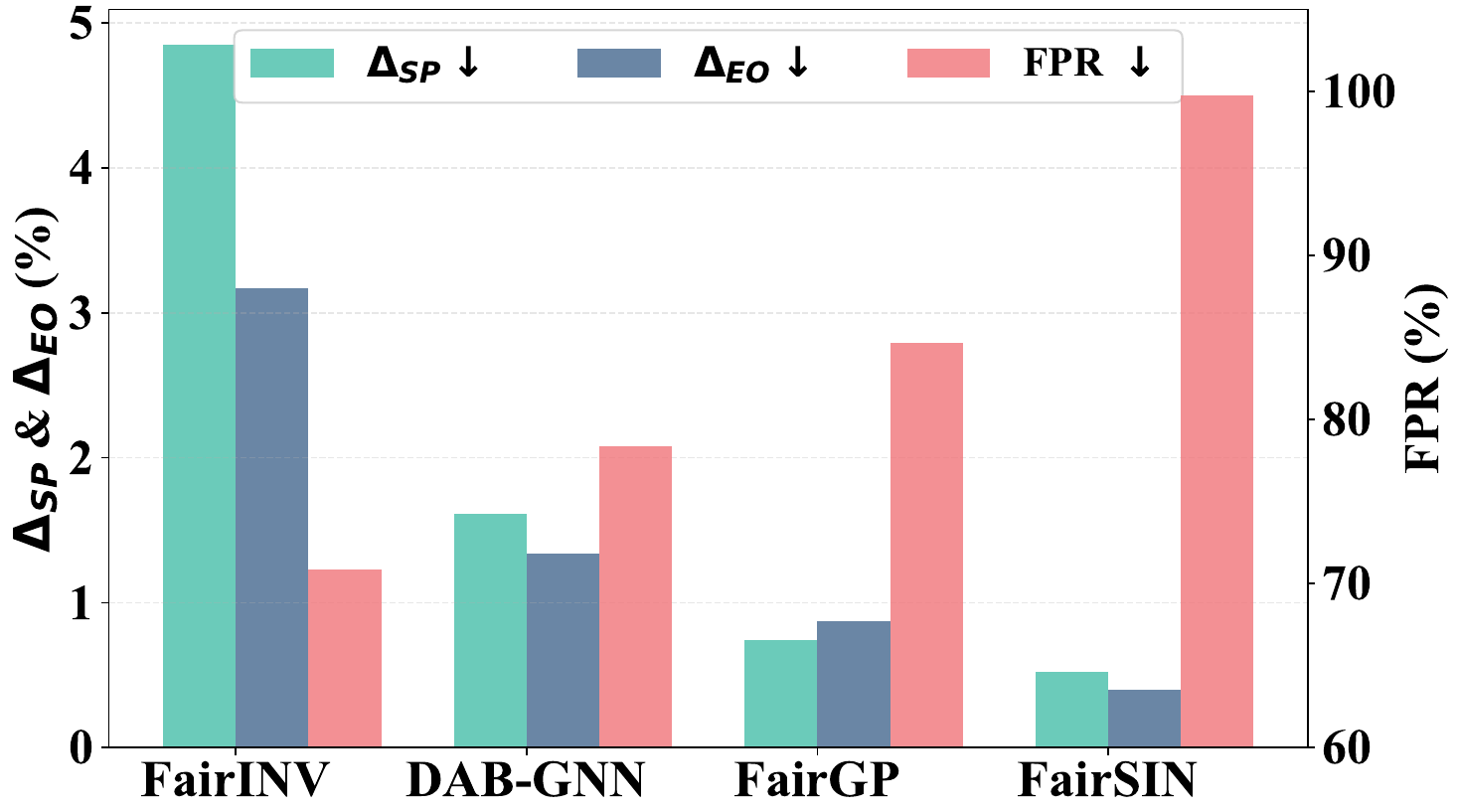}}
\caption{FPR, $\Delta_{SP}$ and $\Delta_{EO}$ of Fairness-aware GNNs on Credit dataset.}
\label{fig:creditFPR}
\end{figure} concerns of GNNs in high-stakes scenarios. Group fairness in GNNs is often evaluated using metrics like Statistical Parity ($\Delta_{\textit{SP}}$) \cite{dwork2012fairness} and Equal Opportunity ($\Delta_{\textit{EO}}$) \cite{hardt2016equality}. Recent state-of-the-art (SOTA) fairness-aware GNNs achieve a reasonable trade-off between accuracy and group fairness.

However, enhancing group fairness such as $\Delta_{\textit{SP}}$ and $\Delta_{\textit{EO}}$ alone may encourage models to adopt a misguided shortcut, called \textit{FPR shortcut}, in which fairness is improved at the cost of excessively lifting the false positive rates (FPR). As shown in Fig.~\ref{fig:creditFPR}, despite mitigating both $\Delta_{\textit{SP}}$ and $\Delta_{\textit{EO}}$, existing fairness-aware methods exhibit extremely high false positive rates in the \textit{Credit defaulter graph}\cite{yeh2009comparisons}. In particular, the FairSIN FPR \cite{yang2024fairsin} is close to 100\% \%, indicating that it classifies nearly all customers as credit card defaulters, which is clearly unreasonable. This phenomenon cannot be captured by $\Delta_{\textit{SP}}$ and $\Delta_{\textit{EO}}$, which primarily focus on the probability of positive predictions within each sensitive group. Fairness-aware GNNs are often deployed in high-stakes scenarios, such as assessing credit risk on the \textit{German credit graph} \cite{dua2017uci} or deciding bail eligibility on the \textit{Recidivism graph} \cite{jordan2015effect}. In these contexts, models with high FPR tend to misclassify high-risk offenders as low-risk or incorrectly assess poor-credit users as reliable, raising serious trustworthy concerns.

In this paper, we pioneer how to improve the fairness of GNNs while avoiding the FPR shortcut. To this end, we formulate the fairness-FPR trade-off as an optimization problem constrained by an upper bound on a two-dimensional structural entropy (2D-SE) objective. Furthermore, we propose FairGSE (\underline{Fair} \underline{G}NN via two-dimensional \underline{S}tructural \underline{E}ntropy), a framework that adaptively reweights graph edges to maximize 2D-SE. FairGSE can provide balanced message aggregation across sensitive groups, discouraging reliance on the FPR shortcut to achieve fairness.

Our contributions are as follows:
\begin{itemize}
    \item  We pioneer an important yet underexplored issue in fairness-aware GNNs: the tendency to improve fairness by excessively lifting high FPR.
    \item We formulate the trade-off between fairness and FPR as an upper-bound optimization problem based on two-dimensional structural entropy. 
    \item We propose FairGSE, a novel GNN framework that adaptively reweights graph edges to maximize the two-dimensional structural entropy objective.
    \item We evaluate FairGSE on multiple real-world datasets. Experimental results demonstrate that FairGSE achieves competitive accuracy and group fairness compared to SOTA methods without significantly lifting FPR.
\end{itemize}

\section{Related Work}

\subsubsection{Fairness in Graph}
Fairness in GNNs is typically studied under two paradigms: \textit{group fairness}~\cite{NIFTY, FairGNN, dong2022structural, cong2023fairsample} and \textit{individual fairness}~\cite{dong2021individual, kang2020inform, zhang2024unraveling}. Group fairness, commonly measured by Statistical Parity ($\Delta_{\textit{SP}}$) \cite{dwork2012fairness} and Equal Opportunity ($\Delta_{\textit{EO}}$) \cite{hardt2016equality}, aims to ensure equal treatment across sensitive groups, and has been the primary focus in fair GNN research. For instance, FairGNN \cite{FairGNN} leverages adversarial learning to minimize reliance on sensitive attributes, while FairVGNN \cite{wang2022improving} and FairINV \cite{zhu2024one} employ dynamic masking and invariant learning to suppress spurious correlations, respectively. Representation-level interventions such as EDITS \cite{dong2022edits} adjust features and structure distributions to mitigate bias, DAB-GNN \cite{lee2025disentangling} further utilizes the Wasserstein distance to constrain the differences between various representations in order to debias. FairSIN \cite{yang2024fairsin} neutralizes sensitive information using a heterogeneous neighbor mean representation estimator. Recently, group fairness has been extended to Graph Transformers~\cite{luo2025fairgp}. Despite these advances, most existing methods focus solely on optimizing fairness metrics (e.g., $\Delta_{\textit{SP}}$ and $\Delta_{\textit{EO}}$). As we show in this work, enforcing group fairness alone can lead to trivial or misleading solutions, such as classifying most nodes as positive, achieving optimal fairness at the cost of a high FPR. 
    
\subsubsection{Structural Entropy}
Information entropy (a.k.a. Shannon entropy) \cite{shannon1948mathematical} is a fundamental measure in information theory that quantifies the uncertainty of a random variable or probability distribution. 
Extending this concept to the graph domain, structural entropy (SE) \cite{li2016structural} is designed to measure the uncertainty associated with the hierarchical partition of graphs.
 Specifically, the nodes of a graph can be partitioned into different groups (e.g., communities or clusters), and each group can be further divided into hierarchical subgroups. This hierarchical partition of the graph can be naturally abstracted as an encoding tree \cite{li2016structural, li2018decoding}, and each tree node represents a specific group. $N$-dimensional structural entropy, which quantifies the number of bits required to encode a graph given an $n$ layers encoding tree, has been increasingly adopted in various graph learning tasks, including adversarial robustness \cite{liu2022residual}, node classification \cite{duan2024structural}, and anomaly detection \cite{zeng2024adversarial,yang2024sebot,xian2025community}. In our work, we leverage two-dimensional structural entropy, which is defined over a one-step partition of nodes based on sensitive attributes, as a principled measure to balance message aggregation across sensitive groups and to avoid FPR shortcut.

\section{Preliminaries}
\subsubsection{Notation}
In this paper, we denote matrices with boldface uppercase letters (e.g., $\mathbf{H}$) and represent the $(i,j)$-th entry of $\mathbf{H}$ as $\mathbf{H}_{(i,j)}$. We denote sets with calligraphic fonts (e.g., $\mathcal{V}$) and the number of elements in the set $\mathcal{V}$ as $|\mathcal{V}|$. Given an undirected attribute graph $G = (\mathcal{V}, \mathbf{A}, \mathbf{X})$, where $\mathcal{V} = \{v_1, v_2, \dots, v_n\}$ is the set of nodes, $\mathbf{X} \in \mathbb{R}^{n \times d}$ is the feature matrix of dimension $d$ and $\mathbf{A} \in \mathbb{R}^{n \times n}$ is the adjacency matrix. For weighted graphs, $\mathbf{A}_{(i, j)}$ denotes the edge weight between $v_i$ and $v_j$; for unweighted graphs, $\mathbf{A} \in \{0,1\}^{n \times n}$, where $\mathbf{A}_{(i,j)} = 1$ indicates the existence of an edge $e_{(i,j)}$. We use $y_i$ to denote the ground-truth label of node $v_i$. For simplicity, we follow previous work\cite{FairGNN, wang2022improving, dong2022edits, yang2024fairsin} and assume that each node $v_i$ has a binary sensitive attribute $s_i \in \{0, 1\}$ and a binary ground-truth label $y_i \in \{0, 1\}$, which we define $y_i = 1$ as a \textit{positive} label and $y_i = 0$ as a \textit{negative} label. We use $\mathcal{S} = \{s_1, s_2, \cdots, s_n\}$ and $\mathcal{Y} = \{y_1, y_2, \cdots, y_n\}$ to denote the set of sensitive attributes and ground-truth labels, respectively. A GNN model $\textit{f}$ consists of a GNN encoder $\textit{f}_\theta$ and a linear classifier $c_\phi$. For the node classification task, 
a GNN model $f$ learns the prediction of the label $\hat{y_i} = c_\phi(f_\theta (v_i, G))$ for node $v_i \in \mathcal{V}$. 

\subsubsection{Fairness Metrics}
To quantify group fairness, we follow studies \cite{dwork2012fairness,hardt2016equality} to adopt statistical parity ($\Delta_{SP}$):
\begin{equation}
\label{eq:sp}
    \Delta_{\textit{SP}}=|P(\hat{y}_v=1|s_v=0)-P(\hat{y}_v=1|s_v=1)|,
\end{equation}
where $P(\hat{y}_v=1|s_v=i)$ denotes the probability that a node $v$ with the sensitive attribute $s_v = i$ is predicted as a positive outcome ($\hat{y}_v = 1$), which can be interpreted as the acceptance rate of a sensitive group. And equal opportunity ($\Delta_{EO}$):
\begin{equation}
\label{eq:eo} 
\begin{aligned}
\Delta_{\textit{EO}}=|P(\hat{y}_v=1|s_v=0, y_v=1)\\
- P(\hat{y}_v=1|s_v=1, y_v=1)|,
\end{aligned}
\end{equation}
where $P(\hat{y}_v=1|s_v=i, y_v=1)|$ denotes the probability that a node $v$ with the sensitive attribute $s_v = i$ and the positive label $y_v = 1$ is predicted as a positive outcome ($\hat{y}_v = 1$), which can be considered the true positive rate of a sensitive group. Therefore, $\Delta_{\textit{SP}}$ and $\Delta_{\textit{EO}}$ measure the disparity in acceptance rates (i.e., positive predictions) and true positive rates between different sensitive groups, respectively. Note that a model with lower $\Delta_{\textit{SP}}$ and $\Delta_{\textit{EO}}$ implies better group fairness performance. In the context of node classification, a model that predicts all nodes as positive outcomes implies that both acceptance rate and true positive rate are equal to 1, thereby leading to perfect fairness, i.e., $\Delta_{\textit{SP}}=\Delta_{\textit{EO}}=0$. However, this may result in an excessively high FPR.

\section{Theoretical Analysis}
To address the above problem, we refine the notion of two-dimensional structural entropy proposed in \cite{li2016structural}, which measures the uncertainty of a graph given by the partition of nodes. To incorporate this notion, we formally define the two-dimensional structural entropy w.r.t. sensitive attribute as follows:

\begin{definition}[2D-SE w.r.t. sensitive attribute]
\label{eq:2d-se}
Given a graph $G$, the node set $\mathcal{V}$ of $G$ is partitioned by $\mathcal{P}_S = \{\mathcal{V}_{S_0}, \mathcal{V}_{S_1}\}$, where the \textit{sensitive group} $\mathcal{V}_{S_i}$ denotes a group of nodes with the same sensitive attribute. The two-dimensional structural entropy (2D-SE) w.r.t. sensitive attribute $H^{\mathcal{P}_S}(G)$ of $G$ is:
\begin{equation}
\fontsize{8pt}{10pt}\selectfont
\begin{aligned}
    H^{\mathcal{P}_S}(G) &= -\sum_{\mathcal{V}_{S_i}\in \mathcal{P}_S} \frac{vol(\mathcal{V}_{S_i})}{vol(G)} 
    \sum_{v \in \mathcal{V}_{S_i}} \frac{d_v}{vol(\mathcal{V}_{S_i})} 
    \log \frac{d_v}{vol(\mathcal{V}_{S_i})} \\
    &- \sum_{\mathcal{V}_{S_i}\in \mathcal{P}_S} \frac{g(\mathcal{V}_{S_i})}{vol(G)} 
    \log \frac{vol(\mathcal{V}_{S_i})}{vol(G)},
\end{aligned}
\end{equation}
\end{definition}
where the degree of node $v$ is $d_v= \sum_{u \in N(v)}{\mathbf{A}_{(v,u)}}$ and $N(v)$ is the set of neighbors of $v$. For a sensitive group $\mathcal{V}_{S_i} \in \mathcal{P}_{S}$, the volume $vol(\mathcal{V}_{S_i})$ is defined as the sum of the degrees of all nodes in $\mathcal{V}_{S_i}$. The volume $vol(G)$ denotes the sum of the degrees of all nodes in $G$. $g(\mathcal{V}_{S_i})$ is the sum of the weights of the edges with exactly one endpoint in $\mathcal{V}_{S_i}$.

Our motivation for using $H^{\mathcal{P}_S}(G)$ stems directly from the intuition behind its definition. From the perspective of fairness in GNNs, a higher $H^{\mathcal{P}_S}(G)$ encourages message aggregation across different sensitive groups, preventing GNNs from overfitting to group-specific information. This reduces the dependency of node representations on sensitive attributes and thereby improves fairness. From the perspective of GNN prediction, a higher $H^{\mathcal{P}_S}(G)$ suggests that nodes with different labels within each sensitive group tend to exhibit more uniform degrees, which may help alleviate the issue of insufficient message aggregation for the negative label and thus reduce FPR.

To demonstrate the feasibility of optimizing $H^{\mathcal{P}_S}(G)$ for reducing fairness metrics ($\Delta_{\textit{SP}}$, $\Delta_{\textit{EO}}$) and FPR, we provide the following analysis:

\begin{lemma}[Optimization Feasibility]
\label{lem:1}
Under two assumptions: (1) $\text{vol}(G) > 0$; (2) $\text{vol}(\mathcal{V}_{S_i}) > 0,\ \forall \mathcal{V}_{S_i} \in \mathcal{P}_S$. the 2D-SE $H^{\mathcal{P}_S}(G)$ is differentiable w.r.t.\ every edge $\mathbf{A}_{(i,j)}(>0)$, with:
\begin{align}
\frac{\partial H^{\mathcal{P}_S}(G)}{\partial \mathbf{A}_{(i,j)}} &= \frac{1}{\ln2} \sum_{i\in \{0, 1\}} (\alpha_{S_i} + \beta_{S_i}),
\end{align}
\end{lemma}
where $\alpha_{S_i}$ and $\beta_{S_i}$ are scalar coefficients that are only related to the degrees of the nodes within the sensitive group $\mathcal{V}_{S_i}$, the detailed proof of this lemma is given in Appendix. Hence $H^{\mathcal{P}_S}(G)$ is differentiable and enables gradient-based optimization.

\begin{theorem}[Fairness via 2D-SE bound]\label{thm:1}
Maximizing the 2D-SE of the sensitive attribute partition yields provable upper bounds on both $\Delta_{SP}$ and $\Delta_{EO}$. Formally,
    \begin{align}
        \Delta_{\textit{SP}} \le \sqrt{2(H^{max}-H^{\mathcal{P}_S}(G))}, \\
    \Delta_{\textit{EO}} \le \sqrt{2(H^{max}-H^{\mathcal{P}_S}(G))},
    \end{align}
    where $H^{\max}=\log_2|\mathcal{V}|-\sum_{\mathcal{V}_{S_i} \in \mathcal{P}_S}\frac{|\mathcal{V}_{S_i}|}{|\mathcal{V}|}\log_2\frac{|\mathcal{V}_{S_i}|}{|\mathcal{V}|}$ is the maximum achievable 2D-SE under the partition $\mathcal{P}_S$ when the graph is fully random and every node has uniform degree.
\end{theorem}
\noindent\textbf{Proof Sketch.}
(1) We first rewrite $H^{\mathcal{P}_S}(G)$ as the entropy of a single-step random walk encoding.  (2) Applying the data-processing inequality yields $I(\mathbf{H};\mathcal{S})\leq H^{\max}-H^{\mathcal{P}_S}(G)$, where $I(\mathbf{H};\mathcal{S})$ represents the mutual information between the node representation $\mathbf{H} = f_\theta(\mathcal{V}, G)$ and sensitive attribute $\mathcal{S}$.  
(3) Pinsker-type inequalities bound $\Delta_{\textit{SP}}$ and $\Delta_{\textit{EO}}$ by $\sqrt{2I(\mathbf{H};\mathcal{S})}$.  
Full proofs are deferred to Appendix.
\begin{theorem}[FPR via 2D-SE bound] \label{thm:2}
Let $r = \frac{|\{v | y_v=0\}|}{|\mathcal{V}|}$ be the proportion of negative label. For any classifier $c_\phi$ trained on node representations $\mathbf{H}$:
\begin{equation}
    \text{FPR} \le \frac{r}{1+r} + \sqrt{2(H^{max}-H^{\mathcal{P}_S}(G))}.
\end{equation}
\end{theorem}
\textbf{Proof Sketch.} (1) The FPR is decomposed as a weighted average over sensitive groups conditioned on the negative label. (2) Using the Jensen-Shannon Pinsker inequality, we bound the difference in FPRs between sensitive groups by $H^{\max}-H^{\mathcal{P}_S}(G)$. (3) Under the worst-case configuration where the difference in FPR between sensitive groups is maximal, the overall FPR is upper-bounded by combining this disparity with the ratio $r$ and full proofs are deferred to the Appendix.

\noindent\textbf{Summary}. Theorems \ref{thm:1} and \ref{thm:2} show that optimizing the graph structure to maximize 2D-SE can improve GNN fairness and avoid FPR shortcuts. This encourages more balanced message aggregation from different sensitive groups, and effectively reduces the influence of group imbalance in node representations.

\section{The Proposed Framework: FairGSE}

\begin{figure*}[htbp]
\centerline{\includegraphics[scale=0.28]{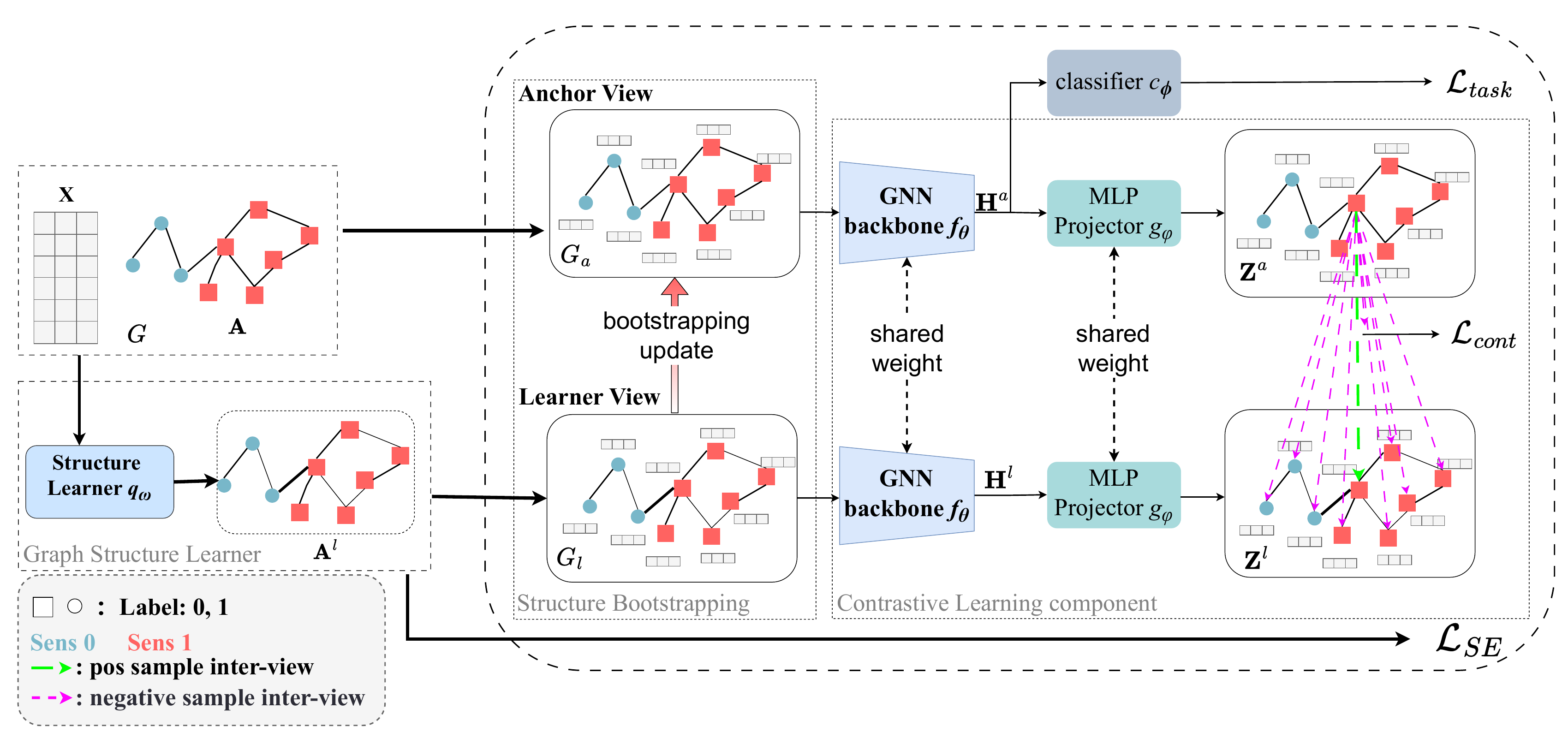}}
\caption{Overview of FairGSE, which consists of graph structure learner, contrastive learning component and structure bootstrapping mechanism.}
\label{fig:FairGSE}
\end{figure*}

\subsection{Overview}

Guided by the theoretical insights above, FairGSE aims to enhance fairness and avoid FPR shortcut by optimizing the 2D-SE of the graph structure, while preserving the critical information in the original graph structure through contrastive learning. An overview of FairGSE is illustrated in Figure~\ref{fig:FairGSE}. First, a graph structure learner $q_\omega$ learns a trainable adjacency matrix $\mathbf{A}^l$ by optimizing edge weights via 2D-SE maximization. Second, a contrastive learning component constructs two graph views: the original graph as the anchor view, and the learner view, defined as the learnable graph $G_l$ formed by $\mathbf{A}^l$. Both views are encoded by a shared GNN, and their projected node representation are aligned using a contrastive loss to preserve structural consistency while promoting fairness. Finally, we employ a structure bootstrapping mechanism to progressively refine the anchor view, aiming to mitigate inherited biases from the original graph and prevent potential overfitting during training.

\subsection{Graph Structure Learner}

The graph structure learner $q_\omega$ in FairGSE learns a learnable graph $G_l$ by optimizing a trainable adjacency matrix $\mathbf{A}^l$, with the goal of maximizing the 2D-SE $H^{\mathcal{P}_{S}}(G_l)$. Initially, $q_\omega$ associates each existing edge $e_{(i,j)}$ in the original graph (where $\mathbf{A}_{(i,j)} > 0$) with a trainable parameter $a_{(i,j)}$. During each training step, these parameters are transformed via a sigmoid function $\sigma(\cdot)$ to produce edge weights in $(0,1)$: $\mathbf{A}^l_{(i,j)} = \sigma(a_{(i,j)})$, thus constructing the weighted adjacency matrix $\mathbf{A}^l$ of $G_l$. The 2D-SE $H^{\mathcal{P}_{S}}(G_l)$ is then computed on $G_l$ according to Definition~\ref{eq:2d-se}. To maximize $H^{\mathcal{P}_{S}}(G_l)$, we minimize the negative entropy objective: \begin{equation} 
\min\ \mathcal{L}_{SE} = -H^{\mathcal{P}_S}(G_l). 
\end{equation} 
Gradients are backpropagated through the computational graph to update the parameters $a_{(i,j)}$. Specifically, the gradient with respect to each $a_{(i,j)}$ is: 
\begin{equation} 
\fontsize{9pt}{10pt}\selectfont
\begin{aligned}
\nabla_{a_{(i,j)}} H^{\mathcal{P}_S}(G_l) = \frac{\partial H^{\mathcal{P}_S}(G_l)}{\partial \mathbf{A}^l} \cdot \frac{\partial \mathbf{A}^l}{\partial \sigma(a_{(i,j)})} \cdot \frac{\partial \sigma(a_{(i,j)})}{\partial a_{(i,j)}}.
\end{aligned}
\end{equation} 

\subsection{Contrastive Learning Component}
We note that only using the 2D-SE may struggle to optimize due to two fundamental challenges: (1) \textit{Optimization instability}. $q_\omega$ updates the entire adjacency matrix $\mathbf{A}^l$ at each training step, continuously changing the graph structure and hindering stable optimization of the GNN encoder. (2) \textit{Structure distortion}. Optimizing the learnable graph $G_l$ may lead to deviations from the original graph structure, which can indirectly degrade the classification performance. To mitigate such deviations, we incorporate a contrastive learning objective.

\noindent\textbf{Graph View Establishment}. \textit{Anchor view} supplies the critical information in the original graph by adopting the original adjacency matrix $\mathbf{A}$. We init anchor view as $G_a=(\mathcal{V}, \textbf{A}^a, \textbf{X})=(\mathcal{V}, \textbf{A}, \textbf{X})$. To provide a more effective fairness enhancement, anchor view is not updated by gradient descent but a structure bootstrapping mechanism, which will be introduced in the next subsection. Unlike standard graph contrastive learning, our \textit{learner view} $G_l$ is optimized by graph structure learner $q_{\omega}$, which enables simultaneous fairness enhancement and representation learning. The learner view is denoted as $G_l=(\mathcal{V}, \textbf{A}^l, \textbf{X})$.

Both views share the same GNN encoder $f_\theta$ and projector $g_\varphi$ (2-layer MLP):
\begin{equation}
    \begin{aligned}
\mathbf{Z}^a &= g_\varphi(f_\theta (\mathcal{V}, G_a))= g_\varphi(\mathbf{H}^a), \\
\mathbf{Z}^l &= g_\varphi(f_\theta (\mathcal{V}, G_l))= g_\varphi(\mathbf{H}^l),
\end{aligned}
\end{equation}
where $\mathbf{H}^a,\mathbf{H}^l$ are the node representation matrices and $\mathbf{Z}^a,\mathbf{Z}^l$ are the projected node representation matrices for anchor view and learner view, respectively.

A symmetric normalized temperature-scaled cross-entropy loss (NT-Xent) \cite{chen2020simple} is then applied to maximize the agreement between the corresponding projected node representation:
\begin{equation} 
   \fontsize{9pt}{10pt}\selectfont
   \begin{aligned}
       \mathcal{L}_{cont}=\frac{1}{2n}\sum_{i=1}^{n}\Bigl[l\Bigl(\mathbf{Z}^{a}_{(i,:)},\mathbf{Z}^{l}_{(i,:)}\Bigr) + l\Bigl(\mathbf{Z}^{l}_{(i,:)}, \mathbf{Z}^{a}_{(i,:)}\Bigr)\Bigr],
    \end{aligned}
\end{equation}
where $\mathbf{Z}^{a}_{(i,:)}$ and $\mathbf{Z}^{l}_{(i,:)}$ denote the $i$-th row of $\mathbf{Z}^{a}$ and $\mathbf{Z}^{l}$, and 
\begin{equation}
 l\Bigl(\mathbf{Z}^{a}_{(i,:)},\mathbf{Z}^{l}_{(i,:)}\Bigr) = \log \frac{e^{sim(\mathbf{Z}^{a}_{i,:},\mathbf{Z}^{l}_{(i,:)})/t}}{\sum_{k=1}^{n}{e^{sim(\mathbf{Z}^{a}_{(i,:)},\mathbf{Z}^{l}_{(k,:)})/t}}},
\end{equation}
$sim(\cdot, \cdot)$ is the cosine similarity function, and $t > 0$ is the temperature parameter. $l\Bigl(\mathbf{Z}^{l}_{(i,:)}, \mathbf{Z}^{a}_{(i,:)}\Bigr)$ is computed following $l\Bigl(\mathbf{Z}^{a}_{(i,:)},\mathbf{Z}^{l}_{(i,:)}\Bigr)$. 

Intuitively, the contrastive loss $\mathcal{L}_{cont}$ is leveraged to enforce maximizing the agreement between the projected representation $\mathbf{Z}^a_{(i,:)}$ and $\mathbf{Z}^l_{(i,:)}$ of the same node $i$ on anchor view and learner view. This joint update maximizes 2D-SE while constraining the deviation of $G_l$ from the original graph structure. Theoretical support for this approach is provided by Theorem 2 in \cite{ling2023learning}, which states that $-\mathcal{L}_{cont} \le I(G_a, G_l)$. This inequality indicates that minimizing the contrastive loss $\mathcal{L}_{cont}$ is equivalent to maximizing the mutual information $I(G_a, G_l)$ between the anchor view and learner view, thereby ensuring that the learned representations capture the essential structural information of the graph.

\subsection{Structure Bootstrapping Mechanism}\label{subsec:bootstrapping}
While fixing the anchor view as the original graph structure $\textbf{A}^a = \textbf{A}$ simplifies contrastive learning, it introduces two barriers to fairness optimization: (1) \textbf{Inherited Bias}. The anchor view inherits biases from the original graph, which fundamentally limits the fairness performance of FairGSE. (2) \textbf{Progressive Overfitting}. Due to the fixed anchor view, the learner view is repeatedly trained to align with the bias hidden within the anchor view, ultimately inheriting similar unfairness.

To address these issues, inspired by \cite{grill2020bootstrap,liu2022towards}, we design a structure bootstrapping mechanism to provide a self-enhancing anchor view. The core idea of our approach is to update the anchor structure $\textbf{A}^a$ with a slow-moving augmentation of the learned structure $\textbf{A}^l$. In particular, given a decay rate $\tau \in [0, 1]$, the anchor structure $\textbf{A}^a$ is updated in epochs as follows:
\begin{equation}
\textbf{A}^{a} = \tau \textbf{A}^{a} + (1-\tau)\textbf{A}^{l},
\end{equation}
where $\tau$ are set to $0.9999$ in this paper. Benefiting from the structure bootstrapping mechanism, FairGSE gradually increases the uncertainty of sensitive attribute in the anchor view by incorporating high 2D-SE from $\textbf{A}^l$ into $\textbf{A}^a$.

\subsection{Training Objective}

Assembling the previously discussed components, the final objective function of FairGSE is depicted in Equation \ref{eq:finalLoss}: 
\begin{equation}
\label{eq:finalLoss}
    	\text{min}\ \mathcal{L} = \mathcal{L}_{task} + \lambda_1 \mathcal{L}_{cont} - \lambda_2 \mathcal{L}_{\textit{SE}}.
\end{equation}

This loss function consists of three parts and is controlled by the tunable hyperparameters $\lambda_1$ and $\lambda_2$ to balance the contributions of the various elements in the overall loss function. Since we focus on the task of node classification, the first term $\mathcal{L}_{\text{task}}$ aims to minimize the classification loss. Specifically, the prediction $\hat{y}_i$ for node $v_i$ is obtained by passing $G_a$ into the GNN encoder $f_\theta$, followed by a linear classifier $c_\phi$, i.e., $\hat{y}_i = c_\phi(f_\theta(v_i, G_a))$. The classification loss is defined as $\mathcal{L}_{\text{task}} = \frac{1}{|\mathcal{V}|} \sum_{i=1}^n [y_i \log(\hat{y}_i) + (1 - y_i)\log(1 - \hat{y}_i)]$.

\begin{table*}[t]
\centering
\setlength{\tabcolsep}{0.68mm}
\begin{tabular}{c|c|cccccccc|c}
\toprule
\textbf{Datasets}                 & \textbf{Metrics} & \textbf{Vanilla GCN}  & \textbf{FairGNN}       & \textbf{EDITS}       & \textbf{FairVGNN}     & \textbf{FairSIN}   & \textbf{FairINV} & \textbf{DAB-GNN} & \textbf{FairGP}  & \textbf{FairGSE} \\
\midrule

\multirow{6}{*}{\textbf{Credit}} & ACC & 75.71$_{0.47}$ & 73.40$_{3.49}$ & 73.35$_{0.40}$ & 70.79$_{0.90}$ & \underline{77.88$_{0.01}$} & 57.42$_{18.73}$ & 77.02$_{1.74}$ & \textbf{77.96$_{7.00}$} & 75.08$_{0.13}$\\                         & AUC & \underline{74.07$_{1.40}$} & 70.26$_{2.94}$ & 71.63$_{0.66}$ & 70.79$_{0.90}$ & 71.99$_{0.86}$ & 71.06$_{2.60}$ & 62.23$_{13.15}$ & 64.23$_{5.82}$ & \textbf{74.53$_{0.39}$}\\
                 & F1  & 83.97$_{0.84}$ & 81.98$_{3.62}$ & 81.77$_{0.45}$ & \textbf{87.60$_{0.07}$} & \underline{87.56$_{0.01}$} & 59.92$_{30.12}$ & 86.11$_{2.21}$ & 87.11$_{0.31}$ & 83.29$_{0.16}$ \\
                 & FPR($\downarrow$)  & 45.96$_{11.38}$ & 46.48$_{21.67}$ & \underline{42.73$_{2.32}$} & 90.70$_{8.44}$ & 99.76$_{0.37}$ & 70.86$_{9.13}$ & 78.36$_{25.23}$ & 84.67$_{9.92}$ & \textbf{41.45$_{1.23}$}  \\
                  & $\Delta_{\textit{SP}} (\downarrow)$ & 14.63$_{4.27}$ & 8.61$_{3.86}$ & 11.54$_{4.81}$ & 3.06$_{2.85}$ & \textbf{0.52$_{0.45}$} & 4.85$_{3.31}$ & 1.61$_{2.19}$ & \underline{0.74$_{0.72}$} & 5.09$_{1.81}$ \\
              & $\Delta_{\textit{EO}} (\downarrow)$ & 12.31$_{4.14}$ & 6.84$_{3.19}$ & 9.66$_{4.81}$ & 1.43$_{1.38}$ & \textbf{0.40$_{0.34}$} & 3.17$_{3.52}$ & 1.34$_{1.71}$ & \underline{0.87$_{0.37}$} & 3.37$_{1.91}$  \\
\cmidrule{1-11}
\multirow{6}{*}{\textbf{Pokec\_n}} & ACC &69.72$_{0.71}$ & \underline{69.75$_{0.84}$} & OOM & 68.52$_{1.00}$ & 67.08$_{1.39}$ & 68.79$_{0.29}$ & 68.12$_{0.30}$ & 64.23$_{1.22}$ & \textbf{71.10$_{0.41}$} \\
                    & AUC & 75.22$_{1.76}$ & \underline{76.30$_{1.44}$} & 
                    OOM & 
                    75.16$_{0.36}$ & 71.39$_{0.89}$ & 73.68$_{0.12}$ & 73.70$_{0.10}$ & 67.43$_{1.99}$ & \textbf{77.24$_{0.15}$} \\
                    & F1  & \textbf{68.08$_{0.70}$} & 64.99$_{1.97}$ & 
                    OOM &
                    67.18$_{0.56}$ & 62.88$_{2.07}$ & 65.30$_{0.45}$ & \underline{67.55$_{0.05}$} & 54.25$_{4.73}$ & 67.22$_{0.70}$ \\
                     & FPR($\downarrow$)  & 27.06$_{3.58}$ & 19.10$_{6.44}$ &
                     OOM & 
                     30.43$_{5.30}$ & 23.92$_{8.89}$ & 23.78$_{0.12}$ & 32.45$_{1.29}$ & \underline{18.98$_{4.10}$} & \textbf{18.46$_{0.34}$} \\
                   & $\Delta_{\textit{SP}} (\downarrow)$ & 1.68$_{1.28}$ & 2.45$_{1.53}$ & 
                   OOM & 
                  \underline{0.94$_{0.84}$} & 1.06$_{0.93}$ & 1.01$_{0.88}$ & 1.96$_{1.56}$ & 1.06$_{0.75}$ & \textbf{0.90$_{0.46}$} \\
                  & $\Delta_{\textit{EO}} (\downarrow)$ &2.44$_{1.22}$ & 2.67$_{1.57}$ &
                  OOM &
                  2.53$_{0.95}$ & 1.99$_{1.02}$ & 2.51$_{1.17}$ & 2.39$_{1.62}$ & \underline{1.93$_{2.00}$} & \textbf{1.38$_{1.10}$}\\
\cmidrule{1-11}
\multirow{6}{*}{\textbf{Pokec\_z}}& ACC &69.62$_{0.77}$ & \underline{70.06$_{0.69}$} & OOM & 61.96$_{4.45}$ & 65.67$_{3.78}$ & 69.51$_{0.23}$ & 68.54$_{0.62}$ & 66.96$_{1.05}$ & \textbf{70.30$_{0.07}$}\\
                  & AUC &75.77$_{1.06}$ & \underline{77.43$_{0.30}$} &
                  OOM &
                  72.57$_{0.82}$ & 73.52$_{0.52}$ & 76.05$_{0.21}$ & 73.85$_{0.32}$ & 72.28$_{2.36}$ & \textbf{78.08$_{0.13}$}\\
                  & F1  & \underline{70.04$_{1.06}$} & 70.05$_{0.78}$ & 
                  OOM &
                  69.00$_{1.77}$ & 68.15$_{1.73}$ & \textbf{70.43$_{0.22}$} & 69.58$_{0.91}$ & 67.95$_{2.71}$ & \underline{70.04$_{0.12}$}\\
                & FPR($\downarrow$)  & 27.56$_{5.96}$ & 27.22$_{6.20}$ & 
                OOM &
                58.96$_{26.78}$ & 38.49$_{22.38}$ & \underline{26.62$_{0.21}$} & 29.92$_{2.92}$ & 32.30$_{11.83}$ & \textbf{22.80$_{0.75}$}\\
               & $\Delta_{\textit{SP}} (\downarrow)$ & 2.28$_{1.75}$ & \underline{1.30$_{0.71}$} & OOM & 2.95$_{0.77}$ & 1.40$_{0.75}$ & 4.75$_{0.70}$ & \textbf{0.96$_{0.73}$} & 3.19$_{1.48}$ & \textbf{0.96$_{0.25}$}\\
                  & $\Delta_{\textit{EO}} (\downarrow)$ & 2.13$_{1.62}$ & 2.57$_{0.84}$ & 
                  OOM & 
                  2.40$_{0.77}$ & \textbf{1.01$_{0.84}$} & 3.50$_{0.50}$ & 3.46$_{0.51}$ & 3.82$_{2.30}$ & \underline{1.38$_{0.11}$}\\
\bottomrule
\end{tabular}
\caption{Comparison results of FairGSE and baseline fairness methods on GCN. In each row, the best result is indicated in \textbf{bold}, while the runner-up result is marked with an \underline{underline}. OOM: out-of-memory on a GPU with 24GB memory.}
\label{tab:comparison}
\end{table*}

\section{Experiments}
 
\begin{table}
\centering
\setlength{\tabcolsep}{1mm}
\begin{tabular}{l|cccl}
\toprule
Dataset                      & Credit        &Pokec\_n         & Pokec\_z        \\ \cmidrule{1-4}
\# Nodes                      & 30,000            & 66,569         & 67,797         \\
\# Features                   & 14            & 266             & 277        \\
\# Edges                      & 2,873,716        & 1,100,663      & 1,303,712       \\
Positive label ratio & 0.78  & 0.51  & 0.54  \\
Sensitive attribute                        & Age        & Region            & Region \\

Avg. degree                  & 95.79        & 16.53       & 19.23    \\
Avg. inter-edge           & 3.84          & 0.73           & 0.90   \\
\bottomrule
\end{tabular}
\caption{Dataset statistics. `Avg.' means `Average number of', `inter-edge' means the edge with the two endpoints that have different sensitive attributes.}
\label{tab:dataset}
\end{table}

In this section, we evaluate and analyze the effectiveness of FairGSE. Specifically, we aim to answer the following questions: \textbf{RQ1:} Does FairGSE outperform baselines in addressing \textit{FPR shortcut}?. \textbf{RQ2:} What is the impact of each component in the FairGSE framework on its overall performance and fairness? \textbf{RQ3:} How sensitive is the performance of FairGSE to the hyperparameters $\lambda_1$ and $\lambda_2$? 

\subsection{Experimental Setup}
\noindent\textbf{Datasets.} Three real-world fairness datasets, namely Credit \cite{yeh2009comparisons}, Pokec\_n and Pokec\_z \cite{takac2012data}, are employed in our experiments. Table \ref{tab:dataset} provides key statistics for these datasets (more details in Appendix).
\noindent\textbf{Baselines.} We compare FairGSE with eight baselines, including Vanilla GCN with two layers and seven SOTA fairness-aware GNN methods: FairGNN \cite{FairGNN}, EDITS \cite{dong2022edits}, FairVGNN \cite{wang2022improving}, FairSIN \cite{yang2024fairsin}, FairINV \cite{zhu2024one}, DAB-GNN \cite{lee2025disentangling} and FairGP \cite{luo2025fairgp}. We use the source codes provided by the authors.
\noindent\textbf{Evaluation Metrics.} To evaluate the utility performance, we use ACC, F1, AUC and FPR as the metrics. To evaluate group fairness, we employ the fairness metrics $\Delta_{\textit{SP}}$ and $\Delta_{\textit{EO}}$.
\noindent\textbf{Implementation Details.}  We use Vanilla GCN as the backbone for both the baselines and FairGSE. Hyperparameter settings for all baseline methods adhere to the guidelines provided by the respective authors. For FairGSE, we use a 2-layer GCN with a hidden layer of size 64, the projector with a 32-dimensional 2-layer MLP and the classifier with a 2-layer MLP. The hyperparameter $\lambda_1$ for contrastive learning is tuned within the range of $\{0.1, 1\}$, $\lambda_2$ for 2D-SE is tuned within $\{0.5, 5\}$. We use the same learning rate and weight decay for all datasets, and they are 0.001 and 1e-5, respectively. Following previous studies, all the methods are evaluated by using a node classification task and splitting the nodes into training(50\% or the default setting of the number of label nodes), validation(25\%), and test (25\%) sets. All evaluations of FairGSE are conducted on a single NVIDIA RTX 3090 GPU with 24GB memory. All models are implemented with PyTorch and PyTorch-Geometric.

\subsection{Comparison Results (RQ1)}

To answer RQ1, we compare FairGSE with eight baseline methods across three real-world datasets. As shown in Table~\ref{tab:comparison}, FairGSE consistently outperforms all baselines in terms of AUC and FPR, while achieving competitive fairness performance.

On the imbalanced Credit dataset (78\% positive labels, cf. Table~\ref{tab:dataset}), FairGSE reduces the FPR by 9.81\% compared to Vanilla GCN, while decreasing $\Delta_{\textit{SP}}$ and $\Delta_{\textit{EO}}$ by 65.21\% and 72.62\%, respectively. When compared to the SOTA fairness method FairSIN, FairGSE improves AUC by 4.88\% and reduces FPR by 58.45\%, albeit with slightly higher $\Delta_{\textit{SP}}$ and $\Delta_{\textit{EO}}$ values. Notably, several baselines achieve near-perfect fairness (e.g., $\Delta_{\textit{SP}}$ and $\Delta_{\textit{EO}}$ close to zero) by exploiting the \textit{FPR shortcut}, as evidenced by their significantly high FPR values.

On the Pokec\_n and Pokec\_z datasets, which have more balanced label, FairGSE achieves the best or second-best performance across all evaluation metrics. For instance, on Pokec\_n, FairGSE achieves the best $\Delta_{\textit{SP}}$ performance and reduces the FPR by nearly 39.34\% compared to the second-best method. 

Overall, Table~\ref{tab:comparison} illustrates that FairGSE achieves the most favorable trade-off between group fairness and FPR control among all evaluated methods. These gains are attributed to its principled joint optimization strategy, which maximizes 2D-SE while preserving critical information in the original graph structure through contrastive learning.

\subsection{Ablation Study (RQ2)}
\begin{figure*}[htbp]
\centerline{\includegraphics[scale=0.458]{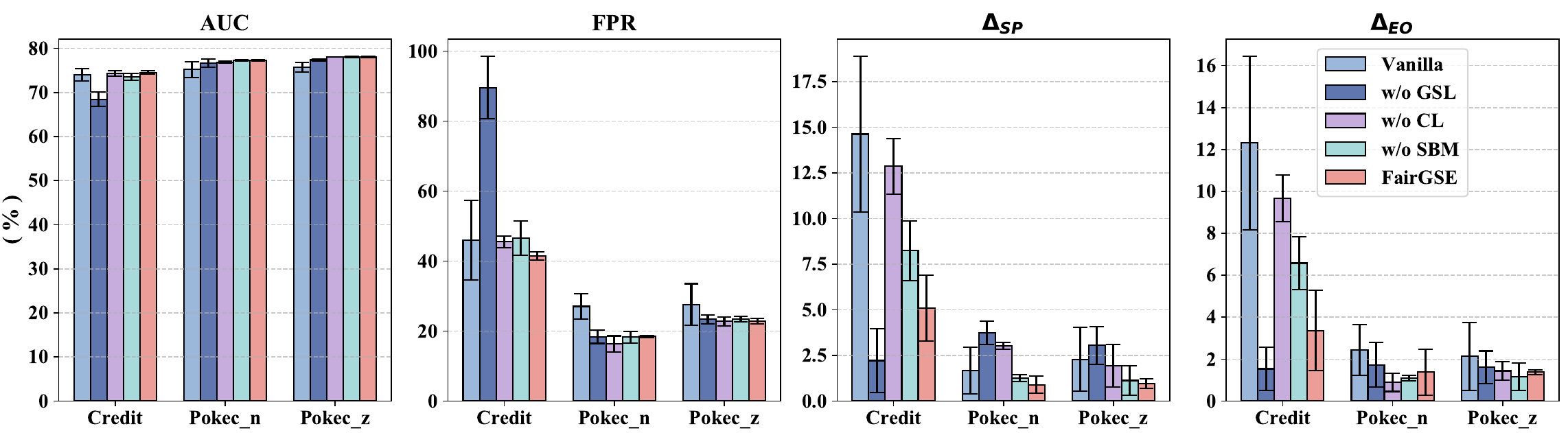}} 
\caption{Ablation study results for FairGSE on all datasets. Higher value indicating better performance for AUC, and lower values preferred for FPR, $\Delta_{\textit{SP}}$ and $\Delta_{\textit{EO}}$.}
\label{fig:ablation}
\end{figure*}

We conduct an ablation study to evaluate the impact of removing individual components from the core modules of FairGSE. Specifically, we consider the following variants: (a) \textbf{FairGSE without Graph Structure Learner (w/o GSL)}: this variant retains only the contrastive learning component. We follow the data augmentation strategy described in GraphCL \cite{you2020graph}, including random edge dropout, feature masking, and node dropping. (b) \textbf{FairGSE without Contrastive Learning (w/o CL)}: this variant removes the contrastive learning objective and trains the model solely by optimizing the adjacency matrix $\textbf{A}$ to maximize 2D-SE, while keeping all other settings unchanged. (c) \textbf{FairGSE without Structure Bootstrapping Mechanism (w/o SBM)}: In this variant, we evaluate the impact of removing the structure bootstrapping mechanism on model performance and fairness. The results across three datasets — Credit, Pokec\_n, and Pokec\_z — are shown in Figure~\ref{fig:ablation}, from which we make the following observations:\\
\noindent\textbullet\ (a) \textbf{FairGSE w/o GSL}: When graph structure learner is removed and only contrastive learning is used, the model exhibits a significantly higher FPR on the imbalanced Credit dataset. However, fairness metrics ($\Delta_{\textit{SP}}$ and $\Delta_{\textit{EO}}$) remain relatively strong. This is due to the fairness-enhancing effect of augmentations used in GraphCL, such as edge drop and feature masking, which have been shown in prior work to mitigate bias \cite{spinelli2021fairdrop}. Nevertheless, fairness strategies alone tend to cause a high FPR on imbalanced data, confirming our earlier analysis.

\noindent\textbullet\ (b) \textbf{FairGSE w/o CL}: When contrastive learning is removed and the model is trained solely via 2D-SE maximization, both FPR and group fairness metrics improve compared to Vanilla GCN. However, performance is less stable and falls short of the optimal trade-off achieved by the full FairGSE. This suggests that while graph structure learner can promote fairness, it continuously altering the graph structure and hindering stable optimization, thereby degrading classification performance. 

\noindent\textbullet\ (c) \textbf{FairGSE w/o SBM}: Our experimental results show that even without the structure bootstrapping mechanism, the model maintains comparable performance metrics to the full FairGSE. However, on the Credit dataset, we observe a notable improvement in fairness when the structure bootstrapping mechanism is included. This is because, without bootstrapping, the structure learner repeatedly aligns with a fixed, potentially biased anchor view. Then, the model is easier to learn the bias within the anchor view.

These findings underscore the importance of FairGSE’s joint optimization strategy: maximizing 2D-SE to promote fairness and reduce FPR while using contrastive learning to preserve the critical information in the original graph structure. As shown in Figure~\ref{fig:ablation}, full FairGSE achieves the highest AUC on Credit while effectively reducing FPR. Moreover, contrastive learning improves training stability and generalization beyond what 2D-SE alone can provide. These results validate our theoretical hypothesis — that entropy-guided re-weighting, combined with contrastive learning, is key to breaking the `fairness–FPR' trade-off.

\subsection{Hyperparameters Analysis (RQ3)}

We explore the sensitivity of FairGSE with respect to two key hyperparameters: $\lambda_1$ and $\lambda_2$. We assess the impact of these hyperparameters on the FairGSE's performance and fairness by conducting experiments where $\lambda_1$ and $\lambda_2$ are varied across the set $\{0.1, 0.2, \cdots, 0.9, 1.0\}$ and $\{0.5,1,\cdots,4.5,5\}$, respectively. We use the imbalanced dataset Credit as a case study, the outcomes are depicted in Figure \ref{fig:parametersens}. The results indicate a clear trend: an increase in $\lambda_1$ generally leads to better fairness, which means FairGSE benefited from the learner view by contrastive loss. This is attributed to the maximum 2D-SE strategy, which enhances the fairness of GNN. Correspondingly, when $\lambda_1$ continuous growth, FairGSE is also affected by the anchor view, whic\begin{figure}[htbp]
\centerline{\includegraphics[scale=0.243]{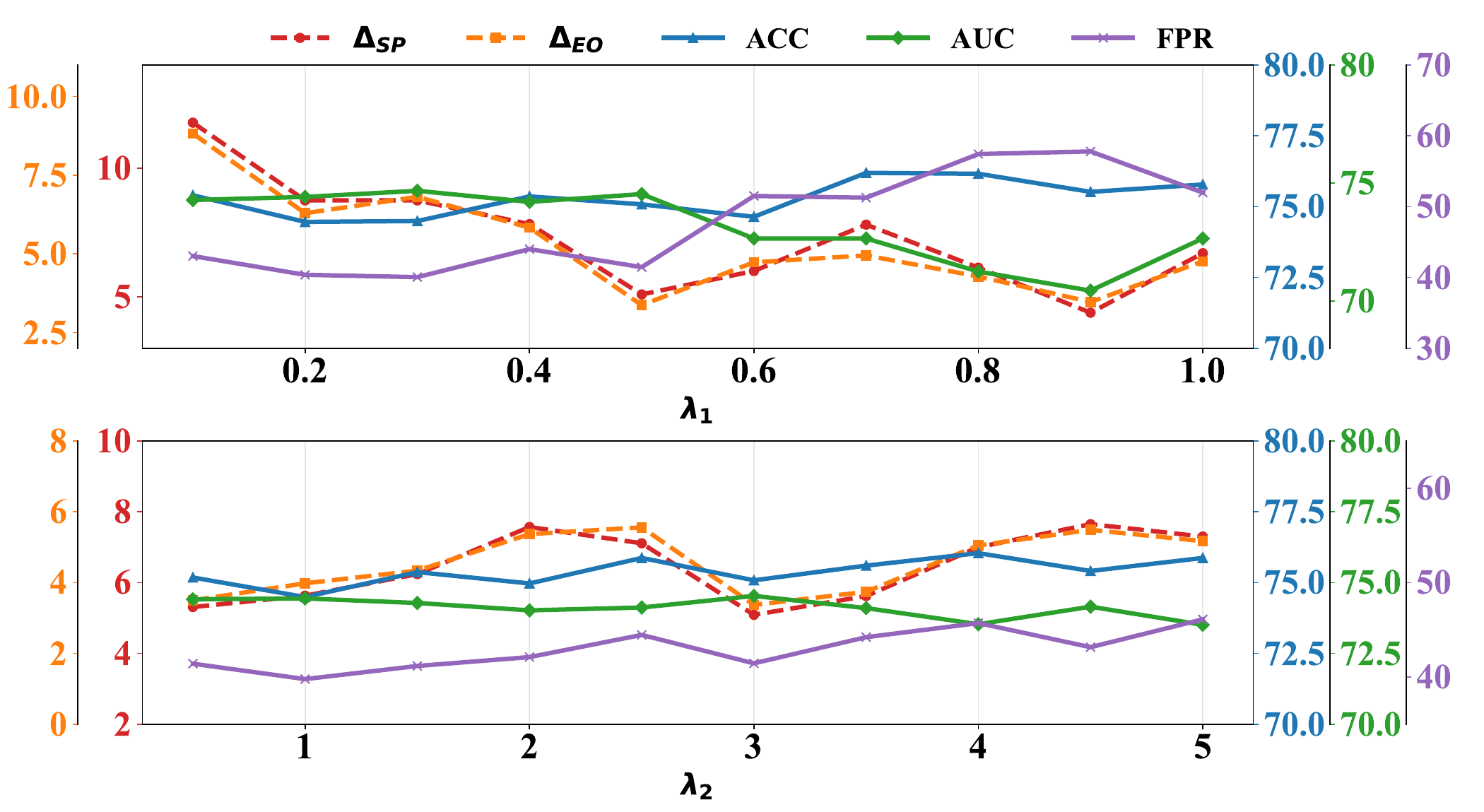}} 
\caption{The hyperparameters study results on the Credit.}
\label{fig:parametersens}
\end{figure}h causes high FPR. When $\lambda_1$ is 0.5, FairGSE achieves the best trade-off of fairness-FPR. The coefficient $\lambda_2$ controls the magnitude of the 2D-SE maximization. Since the objective is to minimise $-\lambda_2\mathcal{L}_{SE}$, if $\lambda_2$ is set too large, the loss term itself becomes dominated by its magnitude, preventing the optimizer from effectively maximizing the 2D-SE. As shown in Figure \ref{fig:parametersens}, the smaller $\lambda_2$ will have a better trade-off of `fairness-FPR'.

\section{Conclusion}
In this work, we investigate the problem of high false positive rates (FPR) in fairness-aware GNNs. To address this challenge, we introduce two-dimensional structural entropy (2D-SE) and theoretically formulate the trade-off between fairness and FPR as a 2D-SE maximization problem. Based on this insight, we propose a novel framework, FairGSE, which leverages contrastive learning to preserve critical structural information from the original graph while optimizing 2D-SE. Experiments demonstrated that FairGSE significantly outperforms existing methods in terms of both fairness metrics ($\Delta_{\textit{SP}}$ and $\Delta_{\textit{EO}}$) and prediction reliability (AUC and FPR). As future work, we aim to extend FairGSE to handle more complex tasks (e.g., multi-class or edge prediction) and the case of limited or unavailable sensitive attributes.

\section{Acknowledgments}
 This work was supported by the National Natural Science Foundation of China (Grant No. U22A2099 and Grant No.62336003).

\bibliography{main}

\clearpage
\appendix
\section{Notations}\label{ap:notations}
In this paper, we denote matrices with boldface uppercase letters (e.g., $\mathbf{H}$) and represent the $(i,j)$-th entry of $\mathbf{H}$ as $\mathbf{H}_{(i,j)}$. We denote sets with calligraphic fonts (e.g., $\mathcal{V}$) and the number of elements in the set $\mathcal{V}$ as $|\mathcal{V}|$. The frequently used notations are listed in Table \ref{tab:notations}.

\begin{table}[th]
\vspace{-3mm}
\centering
\caption{Frequently used notations.}
\vspace{-3mm}
\begin{adjustbox}{width=0.99\columnwidth,center}
\begin{tabular}{l|l}
\toprule
\textbf{Notation} & \textbf{Description} \\ \hline
$\mathcal{G} = (\mathcal{V}, \mathbf{A},\mathbf{X})$  &  The graph.\\ 
$n,d$  &  The number of nodes and the dimension of features.\\ 
$\mathbf{A} \in \mathbb{R}^{n \times n}$ &  The adjacency matrix.\\ 
$\mathbf{X} \in \mathbb{R}^{n \times d}$ &  The feature matrix.\\ 
$\mathcal{G}_l = (\mathcal{V},\mathbf{A}^l,\mathbf{X})$  &  The learner view.\\ 
$\mathbf{A}^l \in \mathbb{R}^{n \times n}$  &  The trainable adjacency matrix.\\ 
$\mathbf{H} \in \mathbb{R}^{n \times d_1}$  &  The node representaion matrix and $d_1$ is the dimension of $\mathbf{H}$.\\ 
$\mathcal{G}_a = (\mathcal{V},\mathbf{A}^a,\mathbf{X})$  & The anchor view.\\
$\mathbf{A}^a \in \mathbb{R}^{n \times n}$ &  The adjacency matrix of anchor view.\\ 
$\mathbf{H}^l,\mathbf{H}^a \in \mathbb{R}^{n \times d_1}$ & The node representation matrix of learner/anchor view. \\
$\mathbf{Z}^l,\mathbf{Z}^a \in \mathbb{R}^{n \times d_2}$ & The projected node representation matrix of learner/anchor view. \\
$d_1,d_2$ & The dimension of node representation/projection. \\
$\mathcal{L}_{cont}$ & The contrastive loss function. \\
$\mathcal{L}_{SE}$ & The two-dimensional structural entropy loss.
\\ \hline
$q_\omega(\cdot)$ & The graph structure learner with parameter $\omega$. \\
$f_\theta(\cdot)$ & The GNN-based encoder with parameter $\theta$. \\
$g_\varphi(\cdot)$ & The MLP-based projector with parameter $\varphi$. \\
$c_\phi(\cdot)$ & The MLP-based classifier with parameter $\phi$. \\
$vol(\cdot)$ & The volume.
\\
$\mathcal{P}_\mathcal{S}$ & The partition of $\mathcal{V}$ based on sensitive attribute $\mathcal{S}$.
\\ 
\bottomrule
\end{tabular}
\end{adjustbox}
\vspace{-5mm}
\label{tab:notations}
\end{table}

\section{Proof}
Throughout the proof procedure, we use the following notation:
\begin{itemize}
  \item \(G=(\mathcal{V},\mathbf{A}, \mathbf{X})\): learnable graph.
  \item \(\mathcal{S}=\{s_1,s_2,\cdots, s_n\}\): the set of sensitive attributes.
  \item \(\mathbf{H}\in\mathbb{R}^{n\times d}\): node embeddings produced by a GNN encoder $f_\theta$.
  \item \(y\in\{0,1\}^{n}\): binary task labels (\(y_i=1\) positive, \(y_i=0\) negative).
  \item \(\mathcal{P}_S=\{\mathcal{V}_{S_0},\mathcal{V}_{S_1}\}\): partition induced by \(\mathcal{S}\).
  \item \(H^{\mathcal{P}_S}(G)\): two-dimensional structural entropy under partition \(\mathcal{P}_\mathcal{S}\).
  \item \(H^{\max}\): $H^{\max}=\log_2|\mathcal{V}|-\sum_{\mathcal{V}_{S_i} \in \mathcal{P}_S}\frac{|\mathcal{V}_{S_i}|}{|\mathcal{V}|}\log_2\frac{|\mathcal{V}_{S_i}|}{|\mathcal{V}|}$ is the maximum achievable 2D-SE under the partition $\mathcal{P}_S$ when the graph is fully random and every node has uniform degree.
  \item $\Delta H\triangleq H^{\max}- H^{\mathcal{P}_S}(\mathbf{A})$: the tow-dimensional structural entropy gap.
\end{itemize}

Given a learnable graph $G=(\mathcal{V}, \mathbf{A}, \mathbf{X})$ with sensitive attribute partition $\mathcal{P}_S = \{\mathcal{V}_{S_0},\mathcal{V}_{S_1}\}$ and an trainable adjacency matrix $\mathbf{A}^l \in \mathbb{R}^{n \times n}$, where $\mathbf{A}^l_{(i,j)} \geq 0$ denotes the edge weight between node $v_i$ and $v_j$. We study group fairness for GNN node classification, where the goal is to learn a model $f: \mathbb{R}^{d} \to \{0, 1\}$ operating on node representations $\mathbf{H}$ such that the predictions satisfy lower $\Delta_{\textit{SP}}$ and $\Delta_{\textit{EO}}$ while minimizing FPR.

\subsection{Proof Lemma 1}
\begin{lemma}[Optimization Feasibility]
\label{ap:lem1}
Under two assumptions: (1) $\text{vol}(G) > 0$; (2) $\text{vol}(\mathcal{V}_{S_i}) > 0,\ \forall \mathcal{V}_{S_i} \in \mathcal{P}_S$. the 2D-SE $H^{\mathcal{P}_S}(G)$ is differentiable w.r.t.\ every edge $\mathbf{A}_{(i,j)}(>0)$, with:
\begin{align}
\frac{\partial H^{\mathcal{P}_S}(G)}{\partial \mathbf{A}_{(i,j)}} &= \frac{1}{\ln2} \sum_{i\in \{0, 1\}} (\alpha_{S_i} + \beta_{S_i}),
\end{align}
\end{lemma}
We derive the gradient of the 2D-SE:
\begin{equation}
    \begin{aligned}
        H^{\mathcal{P}_S}(G) &= -\sum_{\mathcal{V}_{S_i}\in \mathcal{P}_S} \frac{vol(\mathcal{V}_{S_i})}{vol(G)} \sum_{v \in \mathcal{V}_{s}} \frac{d_v}{vol(\mathcal{V}_{S_i})} \log_2 \frac{d_v}{vol(\mathcal{V}_{S_i})} \\
        &- \sum_{\mathcal{V}_{S_i}\in \mathcal{P}_S} \frac{g(\mathcal{V}_{S_i})}{vol(G)} \log_2 \frac{vol(\mathcal{V}_{S_i})}{vol(G)}\\
        &=-\frac{1}{\ln 2}\sum_{\mathcal{V}_{S_i}\in \mathcal{P}_S}\Bigl[
\frac{vol(\mathcal{V}_{S_i})}{vol(G)}H_{\mathcal{V}_{S_i}}^{\text{intra}}
+\frac{g(\mathcal{V}_{S_i})}{vol(G)}\ln\frac{vol(\mathcal{V}_{S_i})}{vol(G)}
\Bigr]
    \end{aligned}
\end{equation}
with respect to any edge weight \(a_{(i,j)}\).
Below, we treat the two terms $\alpha_{S_i}$ and $\beta_{S_i}$ separately.

Let
\[
H_{\mathcal{V}_{S_i}}^{\text{intra}}=-\sum_{v_k\in\mathcal{V}_{S_i}}\frac{d_k}{vol(\mathcal{V}_{S_i})}\ln\frac{d_k}{vol(\mathcal{V}_{S_i})}.
\]
Then $\alpha_{S_i}$ is:
\begin{equation}
    \begin{aligned}
    \frac{\partial}{\partial a_{(i,j)}}\Bigl(\frac{vol(\mathcal{V}_{S_i})}{vol(G)}H_{\mathcal{V}_{S_i}}^{\text{intra}}\Bigr)
    &=\frac{\partial vol(\mathcal{V}_{S_i})}{\partial a_{(i,j)}}\frac{H_{\mathcal{V}_{S_i}}^{\text{intra}}}{vol(G)} \\
    &+\frac{vol(\mathcal{V}_{S_i})}{vol(G)}\frac{\partial H_{\mathcal{V}_{S_i}}^{\text{intra}}}{\partial a_{(i,j)}}\\
    &-\frac{vol(\mathcal{V}_{S_i})}{vol(G)^2}H_{\mathcal{V}_{S_i}}^{\text{intra}}\frac{\partial vol(G)}{\partial a_{(i,j)}},
    \end{aligned}
\end{equation}

where
\[
\frac{\partial vol(\mathcal{V}_{S_i})}{\partial a_{(i,j)}}=\mathbf{1}_{i\in \mathcal{V}_{S_i}}+\mathbf{1}_{j\in \mathcal{V}_{S_i}},\
\frac{\partial vol(G)}{\partial a_{(i,j)}}=2,
\]
and
\begin{equation}
    \begin{aligned}
    \frac{\partial H_{\mathcal{V}_{S_i}}^{\text{intra}}}{\partial a_{(i,j)}}
    &=\sum_{v_k\in {\mathcal{V}_{S_i}}}\Bigl[
    -\frac{(\mathbf{1}_{v_i\in S_i}+\mathbf{1}_{v_j\in S_i})}{vol(\mathcal{V}_{S_i})}\Bigl(1+\ln\frac{d_k}{vol(\mathcal{V}_{S_i})}\Bigr)\\
    &+\frac{d_k}{vol(\mathcal{V}_{S_i})^2}\Bigl(1+\ln\frac{d_k}{vol(\mathcal{V}_{S_i})}\Bigr)
    (\mathbf{1}_{v_i\in \mathcal{V}_{S_i}}+\mathbf{1}_{v_j\in \mathcal{V}_{S_i}})
    \Bigr].
    \end{aligned}
\end{equation}

For $\beta_{S_i}$:
\begin{equation}
    \begin{aligned}
    \frac{\partial}{\partial a_{(i,j)}}\Bigl(\frac{g(\mathcal{V}_{S_i})}{vol(G)}\ln\frac{vol(\mathcal{V}_{S_i})}{vol(G)}\Bigr)
    =\frac{\partial g(\mathcal{V}_{S_i})}{\partial a_{(i,j)}}\frac{1}{vol(G)}\ln\frac{vol(\mathcal{V}_{S_i})}{vol(G)}\\
    +\frac{g(\mathcal{V}_{S_i})}{vol(G)}\Bigl(\frac{1}{vol(\mathcal{V}_{S_i})}\frac{\partial vol(\mathcal{V}_{S_i})}{\partial a_{(i,j)}}
    -\frac{1}{vol(G)}\frac{\partial vol(G)}{\partial a_{(i,j)}}\Bigr),
    \end{aligned}
\end{equation}
with
\[
\frac{\partial g(\mathcal{V}_{S_i})}{\partial a_{(i,j)}}
=\mathbf{1}_{(v_i\in \mathcal{V}_{S_i} \land v_j\notin \mathcal{V}_{S_i})}+\mathbf{1}_{(v_i\notin \mathcal{V}_{S_i}\land v_j\in \mathcal{V}_{S_i})}.
\]

Combining the two contributions yields:
\[
\frac{\partial H^{\mathcal{P}_S}(G)}{\partial a_{(i,j)}}
=\frac{1}{\ln 2}\sum_{\mathcal{V}_{S_i}\in \mathcal{P}_S}[
\alpha_{S_i}
+\beta_{S_i}
].
\]

\subsection{Proof Theorem}
\begin{theorem}[Fairness via 2D-SE bound]\label{thm:fairness_entropy}
Maximizing the two-dimensional structural entropy of the sensitive-attribute partition tightens an information-theoretic upper bound on both Statistical Parity (SP) and Equalized Odds (EO) violations.  Formally,
\begin{align}
\Delta_{\text{SP}} &\le \sqrt{2\,\Delta\mathcal{H}},\\[2pt]
\Delta_{\text{EO}} &\le \sqrt{2\,\Delta\mathcal{H}}.
\end{align}
\end{theorem}

\begin{proof}
\textbf{Step 1: Entropy–Mutual-Information Identity.}\\
Let $\mathcal{U} \sim \pi$ denote the node reached by a stationary random walk on $\mathbf{A}$, and let $S$ be the sensitive attribute. Under the two-dimensional structural entropy model, we have:
\begin{align}
H^{\mathcal{P}_S}(G) &= H(\mathcal{U}, S) \\
&= H(S) + H(\mathcal{U} \mid S).
\end{align}

The maximum achievable 2D-SE under partition $\mathcal{P}_S$, denoted $H^{\max}$, occurs when the graph structure is fully random and all nodes have equal degree (i.e., $\pi(v) = 1/|\mathcal{V}|$ for all $v$). In this case, the conditional distribution $P(\mathcal{U} \mid S)$ becomes uniform within each sensitive group, and we obtain:
\begin{align}
H^{\max} &= H(S) + H(\pi) \\
&= H(S) + \log_2 |\mathcal{V}|.
\end{align}

However, note that this expression assumes uniform population distribution across sensitive groups. In general, when group sizes differ, the correct form of $H^{\max}$ accounts for the group proportions:
\begin{equation}
H^{\max} = \log_2 |\mathcal{V}| - \sum_{\mathcal{V}_{S_i} \in \mathcal{P}_S} \frac{|\mathcal{V}_{S_i}|}{|\mathcal{V}|} \log_2 \frac{|\mathcal{V}_{S_i}|}{|\mathcal{V}|}.
\end{equation}

This can also be written as:
\begin{equation}
H^{\max} = H(\pi) + H(S),
\end{equation}
since $H(\pi) = \log_2 |\mathcal{V}|$ under uniform degree, and $H(S) = -\sum_i p_i \log_2 p_i$ with $p_i = |\mathcal{V}_{S_i}| / |\mathcal{V}|$.

Therefore, the entropy gap is:
\begin{equation}
\begin{aligned}
    H^{\max} - H^{\mathcal{P}_S}(G) &= [H(S) + H(\pi)] - [H(S) + H(\mathcal{U} \mid S)] \\
    &= H(\pi) - H(\mathcal{U} \mid S).
\end{aligned}
\end{equation}

But since $H(\mathcal{U}, S) = H(S) + H(\mathcal{U} \mid S)$ and $H(\mathcal{U}) = H(\pi)$, we have:
\begin{equation}
I(\mathcal{U}; S) = H(\mathcal{U}) - H(\mathcal{U} \mid S) = H(\pi) - H(\mathcal{U} \mid S).
\end{equation}

Thus,
\begin{equation}
H^{\max} - H^{\mathcal{P}_S}(G) = I(\mathcal{U}; S) = \Delta H.
\end{equation}

\textbf{Step 2: Data-Processing Inequality.}\\
The Markov chain $S\to \mathcal{U}\to \mathbf{H}$ gives:
\begin{equation}
    I(\mathbf{H};S)\le I(\mathcal{U};S)=\Delta H.
\end{equation}

\textbf{Step 3: Jensen-Shannon Pinsker Bound.}\\
Let $f:\mathbb{R}^d\to\{0,1\}$ be any classifier $c_\phi$ on $\mathbf{H}$. Using the JS-divergence form of Pinsker’s inequality,
\begin{equation}
    \begin{aligned}
        \Delta_{\text{SP}}&=\bigl|P(f=1\mid S=0)-P(f=1\mid S=1)\bigr| \\
&\le\sqrt{2\,\mathrm{JS}(P_{\mathbf{H}|S=0}\|P_{\mathbf{H}|S=1})}\\
&\le\sqrt{2I(\mathbf{H};S)}
\le\sqrt{2\Delta H}.
    \end{aligned}
\end{equation}
The EO bound follows identically by conditioning on the label $y$.
\end{proof}

\subsection{Proof Theorem 2}\label{thm:2}


\begin{theorem}[FPR via 2D-SE bound]\label{thm:global_fpr}
Let \(r=\frac{|\{v\mid y_v=0\}|}{|\mathcal{V}|}\) be the proportion of negative label. For any classifier $c_\phi$ trained on \(\mathbf{H}\) and with entropy gap \(\Delta\mathcal{H}>0\),
\begin{equation}
    \mathrm{FPR}(f_{\theta})\le\frac{r}{1+r}+\sqrt{2\Delta\mathcal{H}}.
\end{equation}
\end{theorem}

\begin{proof}
\textbf{Step 1: FPR Identity.}\\
\begin{equation}
    \begin{aligned}
        \mathrm{FPR}(f_{\theta})=\Pr_{v\sim \mathcal{V}^-}\!\bigl[f(\mathbf{H}_v)=1\bigr]\\
=\sum_{s\in\{0,1\}}\Pr(S=s\mid Y=0)\cdot\Pr(f_{\theta}=1\mid S=s,Y=0).
    \end{aligned}
\end{equation}
where $\mathcal{V}^-$ is the set of nodes with negative label in graph $G$.

\textbf{Step 2: Pinsker Bound via JS Divergence.}\\
Let $E = \{f = 1\}$ be the event that the GNN model $f$ predicts a node as positive outcome. Then,
\begin{equation}
    \begin{aligned}
        \bigl|\Pr(E\mid S=0)-\Pr(E\mid S=1)\bigr| &\le\sqrt{2\,\mathrm{JS}(P_{\mathbf{H}|S=0}\|P_{\mathbf{H}|S=1})}\\
        &\le\sqrt{2I(\mathbf{H};S)}\le\sqrt{2\Delta H}.
    \end{aligned}
\end{equation}

\textbf{Step 3: Combine Bounds.}\\
Define $p_s = \Pr(E \mid S = s, Y = 0)$ for $s \in \{0,1\}$. Then,
\begin{equation}
    \mathrm{FPR}(f) = \sum_{s \in \{0,1\}} \Pr(S = s \mid Y = 0) \cdot p_s \le \max(p_0, p_1).
\end{equation}
Using the identity
\begin{equation}
    \max(p_0, p_1) \le \frac{p_0 + p_1}{2} + \frac{|p_0 - p_1|}{2},
\end{equation}
and applying the JS-Pinsker bound from Step 2 yields
\begin{equation}
    \mathrm{FPR}(f) \le \frac{p_0 + p_1}{2} + \frac{\sqrt{2\Delta H}}{2}.
\end{equation}
Now, under the worst-case assumption on $p_0$ and $p_1$, we have $\frac{p_0 + p_1}{2} \le \frac{r}{1 + r}$, where $r = \Pr(Y = 0)$ is the global negative-class proportion. This follows from the fact that $\Pr(E \mid Y = 0) \le 1$, and the class imbalance implies an upper bound on the average prediction rate.

Therefore, we conclude:
\begin{equation}
    \mathrm{FPR}(f) \le \frac{r}{1 + r} + \sqrt{2\Delta H}.
\end{equation}
\end{proof}

\section{Experiments}
\subsection{Datasets}
Three real-world fairness datasets, namely Credit, Pokec\_n and Pokec\_z, are employed in our experiments. We give a brief overview of these datasets as follows:
\begin{itemize}
    \item Credit \cite{yeh2009comparisons} is a credit card users dataset. Nodes are credit card users and they are connected based on the pattern similarity of their purchases and payments. Considering "age" as the sensitive attribute, the task is to predict whether a user will default on credit card payment.

    \item Pokec\_n/z \cite{FairGNN} are collected from a popular social network in Slovakia, where Pokec\_n and Pokec\_z are social network data in two different provinces. Nodes denote users and edge represents the friendship between users. Considering “region” as the sensitive attribute, the task is to predict the working field of the users.
    
\end{itemize}

\subsection{Baselines}
We compare FairGSE with seven SOTA fairness-aware GNNs, including FairGNN, EDITS, FairVGNN, FairSIN, FairINV, DAB-GNN and FairGP. A brief overview of these methods is shown as follows:
\begin{itemize}
    \item FairGNN \cite{FairGNN} aims to learn fair GNNs with sensitive attribute information by a sensitive attribute estimator and adversarial learning.

    \item EDITS \cite{dong2022edits} minimize the Wasserstein Distance of node features and graph topology between groups. 


    \item FairVGNN \cite{wang2022improving} preventing sensitive attribute leakage using adversarial learning and clamping weights.

    \item FairSIN \cite{yang2024fairsin} neutralizes sensitive information using a heterogeneous neighbor mean representation estimator.

    \item FairINV \cite{zhu2024one} is a method which combines unsupervised inference of sensitive attributes with invariant learning to eliminate spurious correlations between sensitive attributes and labels.

    \item DAB-GNN \cite{lee2025disentangling} disentangles attribute, structure, and potential biases via dedicated disentanglers, amplifies each bias to preserve its characteristics, and then debiases through distribution alignment across subgroups by utilizing the Wasserstein distance.

    \item FairGP \cite{luo2025fairgp} mitigates sensitive-feature bias in global attention by graph partitioning and eliminating inter-cluster attention, thereby achieving scalable and fair Graph Transformers.
\end{itemize}





\end{document}